\documentclass[11pt,letterpaper]{article}
\usepackage[letterpaper, total={6.5in, 9in}]{geometry}
\usepackage{amsmath}
\usepackage{amsfonts}
\usepackage{color}
\usepackage{latexsym}
\usepackage{tablefootnote}
\usepackage{epsfig}
\usepackage{multirow}
\usepackage{float}
\usepackage{array}
\usepackage{amssymb,amsthm}

\newcommand*{\QEDA}{\hfill\hbox{\vrule width1.0ex height1.0ex}}
\usepackage{algorithm}
\usepackage{algorithmic}
\usepackage{hyperref}

\usepackage{makecell,booktabs}
\usepackage[version=4]{mhchem}
\usepackage[strict]{changepage}

\usepackage{soul}

\usepackage{amsmath}
\usepackage{amsfonts}
\usepackage{latexsym}
\usepackage{amssymb}

\newtheorem{thm}{Theorem}[section]
\newtheorem{theorem}[thm]{Theorem}

\newtheorem{lemma}[thm]{Lemma}

\newtheorem{proposition}[thm]{Proposition}

\newtheorem{definition}[thm]{Definition}

\newtheorem{remark}[thm]{Remark}

\newcommand{\beq}{\begin{equation}}
\newcommand{\eeq}{\end{equation}}
\newcommand{\beqa}{\begin{eqnarray}}
\newcommand{\eeqa}{\end{eqnarray}}
\newcommand{\beqas}{\begin{eqnarray*}}
\newcommand{\eeqas}{\end{eqnarray*}}
\newcommand{\bi}{\begin{itemize}}
\newcommand{\ei}{\end{itemize}}

\newcommand{\vgap}{\vspace{.1in}}
\newcommand{\nn}{\nonumber}

\newcommand{\R}{\mathbb{R}}

\newcommand{\lam}{{\lambda}}

\newcommand{\inner}[2]{\langle #1,#2\rangle}

\newcommand{\argmin}{\mathrm{argmin}\,}

\newcommand{\Argmin}{\mathrm{Argmin}\,}

\newcommand{\tx}{\tilde x}

\begin{document}
	\title{A Proximal Algorithm for Sampling from Non-smooth Potentials}
	\date{Feburary 9, 2021}
	\author{
		Jiaming Liang \thanks{School of Industrial and Systems
			Engineering, Georgia Institute of
			Technology, Atlanta, GA, 30332.
			(email: {\tt jiaming.liang@gatech.edu}).}
		\qquad
		Yongxin Chen \thanks{School of Aerospace Engineering, Georgia Institute of
			Technology, Atlanta, GA, 30332. (email: {\tt yongchen@gatech.edu}).}
	}
	\maketitle
	
	\begin{abstract}
	
	In this work, we examine sampling problems with non-smooth potentials. We propose a novel Markov chain Monte Carlo algorithm for sampling from non-smooth potentials. We provide a non-asymptotical analysis of our algorithm and establish a polynomial-time complexity $\tilde {\cal O}(d\varepsilon^{-1})$ to obtain $\varepsilon$ total variation distance to the target density, better than most existing results under the same assumptions. Our method is based on the proximal bundle method and an alternating sampling framework. This framework requires the so-called restricted Gaussian oracle, which can be viewed as a sampling counterpart of the proximal mapping in convex optimization. One key contribution of this work is a fast algorithm that realizes the restricted Gaussian oracle for any convex non-smooth potential with bounded Lipschitz constant.\\


	\end{abstract}
	
\section{Introduction}\label{sec:intro}


Core to many scientific and engineering problems that face uncertainty (either physically or algorithmically) is the task of drawing samples from a given, often unnormalized, probability density. Sampling plays a crucial role in many applications such as statistical inference/estimation, operations research, physics, biology, and machine learning, etc \cite{bertsimas2004solving,durmus2018efficient,dyer1991random,gelman2013bayesian,kalai2006simulated,kannan1997random,krauth2006statistical,sites2003delimiting}. For instance, in Bayesian statistics, we can sample from the posterior distribution to infer its mean, covariance, or other important statistics. Sampling is heavily used in molecular dynamics to discover new structures. Sampling is also closely related to optimization. On the one hand, optimization can be viewed as the limit of sampling when the temperature parameter goes to 0. On the other hand, sampling can be viewed as an optimization over the manifold of probability distributions \cite{wibisono2018sampling,yang2020variational}.

Over the years, many methods and algorithms have been developed for sampling \cite{applegate1991sampling,chen2020fast,dyer1991random,lee2021structured,lee2017eldan,lovasz2006fast,NEURIPS2019_eb86d510}. A very popular framework for sampling from high dimensional complex distributions is the Markov chain Monte Carlo (MCMC) algorithm \cite{chen2018fast,cheng2018convergence,cheng2018underdamped,durmus2019analysis,durmus2018efficient}. In MCMC, a Markov chain is constructed so that its invariant distribution is the given target distribution we want to sample from. After running the Markov chain for sufficiently many iterations, the state will follow this invariant distribution, generating samples from it. Several widely used MCMC methods include Langevin Monte Carlo (LMC) \cite{dalalyan2017theoretical,grenander1994representations,parisi1981correlation,roberts1996exponential}, Metropolis-adjusted Langevin algorithm (MALA) \cite{bou2013nonasymptotic,roberts2002langevin,roberts1996exponential}, and Hamiltonian Monte Carlo (HMC) \cite{neal2011mcmc}. These three algorithms use gradient information of the potential (log-density) to construct the Markov chain. They resemble the gradient-based algorithms in optimization and can be viewed as the sampling counterparts of them. Over the last few years, many theoretical results (see \cite{chen2020fast,dalalyan2017theoretical,durmus2019analysis,dwivedi2018log,lee2021structured,lee2017eldan,roberts1998optimal,roberts1996exponential} and references therein) have been established to understand the computational complexities of these MCMC algorithms. 

Most existing gradient-based MCMC methods are only applicable to settings with smooth potentials \cite{dalalyan2017theoretical,lee2021structured,wibisono2018sampling} whose gradient is Lipschitz continuous. However, non-smooth sampling is also an important problem as many applications of sampling involve non-smooth potentials. For instance, in Bayesian inference, the prior is naturally non-smooth when a compact support is considered. Many problems in deep learning are also non-smooth, not only due to non-smooth activation functions like ReLU used in the neural networks, but also due to intrinsic scaling symmetries. Nevertheless, the study of sampling without smoothness is nascent. This is in sharp contrast to optimization where a plethora of algorithms, e.g., subgradient method, proximal algorithm, bundle method have been developed for non-smooth optimization \cite{lemarechal1975extension,lemarechal1978nonsmooth,liang2020proximal,liang2021unified,mifflin1982modification,rockafellar1976monotone,wolfe1975method}.

{\bf Our contributions.} The goal of this work is to establish an efficient algorithm to draw samples from a distribution with non-smooth potential. We focus on the case where the potential is convex and is Lipschitz continuous. Our algorithm is based on the recent alternative sampling framework (ASF) \cite{lee2021structured}, which can be viewed as a sampling counterpart of the proximal point method in optimization \cite{rockafellar1976monotone}. The key of the alternative sampling framework is a step known as the restricted Gaussian oracle (RGO) (see Definition \ref{def:RGO}) to draw samples from a potential regularized by a large isotropic quadratic term. To utilize this framework to sample from general non-smooth potentials with bounded Lipschitz constants, we develop an efficient realization of the RGO through rejection sampling with a properly designed proposal. A non-smooth optimization technique known as the proximal bundle method \cite{lemarechal1975extension,liang2020proximal,liang2021unified} is used to compute the proposal. We establish a polynomial-time complexity $\tilde {\cal O}(d \varepsilon^{-1})$ to obtain $\varepsilon$ total variation distance to the target density, better than most existing results under the same assumptions.

A key contribution of this paper
is a fast algorithm for implementing RGO for any convex (either smooth or non-smooth) function.
When the potential $g$ is decomposable, e.g., $g$ is an $\ell_1$ norm or an indicator function of an orthant, there exists simple sampling algorithms for RGO \cite{mou2019efficient}.
In general, the implementation of RGO is a difficult algorithmic task, which makes ASF \cite{lee2021structured} a conceptual method without implementable algorithms in some cases.
Our algorithm of implementing the RGO for any convex function broads the applicability of the ASF significantly. In fact, our algorithm for RGO can be used for any framework, not only ASF, that requires sampling from $\exp(-g(x)-\frac{1}{2\eta}\|x-y\|^2)$ for any $y$ and some proper $\eta>0$. From an optimization point of view, our algorithm for RGO provides an efficient realization of the proximal oracle for a wide range of functions/potentials, solidifying the connections between the ongoing research at the interface of optimization and sampling \cite{wibisono2018sampling}.

{\bf Related Work.} Over the last few years, several new algorithms and theoretical results in sampling with non-smooth potentials have been established. In \cite{mou2019efficient}, sampling for non-smooth composite potentials is considered. The algorithm needs the proximal sampling oracle that samples from the target potential regularized by a large isotropic quadratic term as well as computes the corresponding partition function, which is not realistic for general potentials. In \cite{shen2020composite}, algorithms to sample from non-smooth composite potentials are developed; both are based on the RGO which is similar to the proximal sampling oracle but do not need to compute the partition function. In \cite{chatterji2020langevin}, the authors developed an algorithm to sample from non-smooth potentials by running LMC on the Gaussian smoothing of the potentials. In \cite{lehec2021langevin}, the author developed the projected LMC algorithm and analyzed its complexity for non-smooth potentials. In \cite{freund2021convergence}, the authors developed a new analysis that leads to dimension-free complexity for sampling from a composite density which contains a non-smooth component. In \cite{durmus2019analysis}, the authors presented an optimization approach to analyze the complexity of sampling and established a complexity result for sampling with non-smooth composite potentials. In \cite{lee2017eldan}, the authors studied the complexity of the ball walk to sample from an isotropic logconcave density from a warm start. In \cite{bernton2018langevin,wibisono2019proximal}, a proximal algorithm was proposed. This algorithm resembles the ASF for sampling with a major difference that the RGO is replaced by proximal point optimization step, which introduces bias for sampling. 

To compare our results with \cite{bernton2018langevin} and \cite{freund2021convergence}, consider sampling from $\exp(-f(x)-\mu\|x\|^2/2)$ where $f$ is convex and $M$-Lipschitz continuous. Our complexity (see Theorem \ref{thm:g}) is $\tilde {\cal O}(M^2 d/\mu)$, better than ${\cal O}(M^2/(\mu \varepsilon^2))$ in \cite{freund2021convergence} and $\tilde {\cal O}(M^2 d/(\mu\epsilon^4))$ (albeit in Wasserstein distance) \cite{bernton2018langevin} when $\varepsilon<d^{-1/2}$.
For sampling from non-smooth potentials, compared with \cite{mou2019efficient,shen2020composite}, our algorithm does not require any sampling oracle.
Compared with \cite{lehec2021langevin}, we consider sampling from a distribution supported on $\R^d$ instead of a convex compact set.
Compared with \cite{chatterji2020langevin,durmus2019analysis}, our algorithm has better complexity in terms of total variation when the target error $\varepsilon$ is small.
Compared with \cite{lee2017eldan}, we consider a generic setting where the target distribution can be anisotropic, and our complexity is in general better in the low resolution region.
See Table \ref{tab:t1} for the detailed complexity bounds. Note that complexity results obtained in \cite{chatterji2020langevin,lee2017eldan} and this paper are for the last iterate, while the bound established in \cite{durmus2019analysis} (also \cite{freund2021convergence}) is for the average of all iterates.
\begin{table}[H]
	\begin{centering}
		\begin{tabular}{|>{\centering}p{2cm}|>{\centering}p{3.2cm}|>{\centering}p{3cm}|>{\centering}p{3cm}|>{\centering}p{3cm}|}
			\hline 
			{Paper} & {\cite{chatterji2020langevin}}  & {\cite{durmus2019analysis}} & {\cite{lee2017eldan}} & {this paper} \tabularnewline
			\hline 
			{Complexity} & {$\tilde {\cal O} (M^6 d^5 {\cal M}_4^{3/2} \varepsilon^{-10})$} & {$ {\cal O} (M^2 W_2^2 \varepsilon^{-4})$} & {$ {\cal O} (d^{5/2} \log(\beta/\varepsilon) )$} & {$\tilde {\cal O} (M^2 d {\cal M}_4^{1/2} \varepsilon^{-1})$} \tabularnewline
			\hline
		\end{tabular}
		\par\end{centering}
	\caption{Complexity bounds for sampling from non-smooth densities.}\label{tab:t1}
\end{table}
In Table \ref{tab:t1}, ${\cal M}_4$ denotes the finite fourth moment of the target distribution, $W_2$ denotes the Wasserstein distance between the initial and target distributions, and $\beta$ denotes the warmness of the initial distribution. More specifically, ${\cal M}_4 \approx d^2$ in the isotropic case, $\log\beta\approx d$ if the initial distribution is not warm started, and $W_2\approx \sqrt{d}$ in general. Under these simplifications, our bound is $\tilde {\cal O} (M^2 d^2\varepsilon^{-1})$, better than $\tilde {\cal O} (M^6 d^8\varepsilon^{-10})$ in \cite{chatterji2020langevin}, $ {\cal O} (M^2 d \varepsilon^{-4})$ in \cite{durmus2019analysis}, and $\tilde {\cal O} (d^{7/2})$ in \cite{lee2017eldan} when $d$ is large and $\epsilon$ is relatively small. Note that, for typical problems, $M = {\cal O}(d^{1/2})$, but sparsity maybe exploited to improve this dependence.

{\bf Organization.} The rest of this paper is structured as follows. In Section \ref{sec:formulation} we provide the problem formulation we are interested in. We also briefly review ASF on which our algorithm is based. In Section \ref{sec:RGO} we present our key contribution, an efficient realization of the RGO for general convex potentials with bounded Lipschitz constants. This is then combined with the alternating sampling framework in Section \ref{sec:main} to establish our results for sampling without smoothness. 
In addition, we further apply the proposed algorithm to sample from smooth densities in Section \ref{sec:smooth}, and establish iteration-complexity bounds.
Finally, we present some concluding remarks and possible extensions in Section \ref{sec:conclusion}.

\section{Problem formulation and alternating sampling framework}
\label{sec:formulation}

The problem of interest is to sample from a distribution on $\R^d$ proportional to $\exp(-f(x))$ where the potential $f$ is {\bf convex and $M$-Lipschitz continuous}. Note that the potential $f$ does not need to be smooth. This violates the smoothness assumption for most existing gradient-based MCMC sampling methods \cite{lee2021structured,wibisono2018sampling}.
Before describing our approach to design an efficient algorithm for sampling with non-smooth potentials, we introduce two important algorithmic notions used in this paper.

Our method is built on the alternating sampling framework (ASF) introduced in \cite{lee2021structured} (a similar method was developed in \cite{vono2022efficient}), which is a generic framework for sampling from a distribution $\exp(-g(x))$; ASF is itself a special case of Gibbs sampling.
For a given point $x\in \R^d$ and stepsize $\eta>0$, the alternating sampling framework repeats the two steps as in Algorithm \ref{alg:ASF}.

\begin{algorithm}[H]
	\caption{Alternating Sampling Framework \cite{lee2021structured}}
	\label{alg:ASF}
	\begin{algorithmic}
		\STATE 0. sample $y\sim \pi_x(y) \propto \exp(-\frac{1}{2\eta}\|x-y\|^2)$
		\STATE 1. sample $x\sim \pi_y(x) \propto \exp(-g(x)-\frac{1}{2\eta}\|x-y\|^2)$
	\end{algorithmic}
\end{algorithm}

In Algorithm \ref{alg:ASF}, sampling $y$ given $x$ in step 1 can be easily done since $\pi_x(y) = {\cal N}(x,\eta I)$.
Sampling $x$ given $y$ in step 2 corresponds to the so-called restricted Gaussian oracle for $g$ introduced in \cite{lee2021structured}, which is the second crucial algorithmic notion used in this paper.
\begin{definition}\label{def:RGO}
	Given a point $y\in \R^d$ and stepsize $\eta >0$, a restricted Gaussian oracle (RGO) for convex $g:\R^d\to \R$ is a sampling oracle that returns a sample from a distribution proportional to $\exp(-g(\cdot) - \|\cdot-y\|^2/(2\eta))$.
\end{definition}

RGO is an analogy of the proximal mapping in convex optimization, which is heavily used in proximal point methods.
RGO is a key algorithmic ingredient used in \cite{lee2021structured} together with the alternating sampling framework to improve the iteration-complexity bounds for various sampling algorithms.
Examples of a convex function $g$ that admits an computationally efficient RGO have been presented in \cite{mou2019efficient,shen2020composite}, including coordinate-separable regularizers, $\ell_1$-norm, and group Lasso. 

We recall the main result of \cite{lee2020structured,lee2021structured}, which gives the complexity of Algorithm \ref{alg:ASF} in terms of number of calls to RGO and is useful in this paper.

\begin{theorem}\label{thm:outer}
	(Theorem 1 of \cite{lee2020structured})
	Let $\pi$ be a distribution on $\R^d$ with $\pi(x) \propto \exp \left(-f_{\text {oracle }}(x)\right)$ such that $f_{\text {oracle }}$ is $\mu$-strongly convex, and let $\varepsilon \in(0,1)$. Let $\eta\le 1/\mu, T=\Theta\left(\frac{1}{\eta \mu} \log \frac{d}{\eta \mu \varepsilon}\right)$. Algorithm \ref{alg:ASF}, initialized at the minimizer\footnotemark \footnotetext{Actually, this minimizer can be an approximate solution, as long as $\|x-x_\text{opt}\|^2\le d/\mu$ where $x$ and $x_\text{opt}$ are the approximate and exact solutions, respectively.} of $f_{\text{oracle}}$, runs in $T$ iterations, each querying RGO for $f_{\text {oracle}}$ with parameter $\eta$ a constant number of times, and obtains $\varepsilon$ total variation distance to $\pi$.
\end{theorem}

We are now ready to describe our approach to sample from non-smooth potentials.
We first consider a regularized density $\exp(-g(x))$ where $g(x)=f(x)+\mu\|x-x^0\|^2/2$ and $x^0\in\R^d$ is  an arbitrary point but preferred to be close to the minimum set of $g$. Since $g$ is $\mu$-strongly convex, the alternating sampling framework is applicable to it. 
We then develop an efficient implementation of the RGO based on the proximal bundle method and rejection sampling for an arbitrary convex potential with bounded Lipschitz constant. This is the main contribution of this work.
Finally, we justify the sample generated from $\exp(-g(x))$ by using Algorithm \ref{alg:ASF} and the implementable RGO with a proper choice of $\mu$ is a sample within $\varepsilon$ total variation distance to the target density $\exp(-f(x))$.

\section{Key result: an implementable restricted Gaussian oracle}\label{sec:RGO}
The bottleneck of applying the alternating sampling framework (Algorithm \ref{alg:ASF}) to sample from general log-concave distributions is the availability of RGO.
In this section, we focus on designing computationally efficient and implementable RGO for $\mu$-strongly convex $g$ of the form 
\begin{equation}\label{eq:g}
	g=f+\mu \|\cdot-x^0\|^2/2.
\end{equation}

Subsection \ref{subsec:opt-oracle} presents an implementation of the RGO for $g$ by using the proximal mapping of $f$ and rejection sampling.
In order to develop an implementable RGO for $g$ in the cases where there are no efficient optimization oracles for $f$, 
we resort to the proximal bundle method, which is a standard method in convex non-smooth optimization.
Subsection \ref{subsec:bundle} briefly reviews the proximal bundle method and its iteration-complexity bound.
Finally, Subsection \ref{subsec:no opt-oracle} describes the implementation of RGO for $g$ based on rejection sampling and the proximal bundle method, instead of the proximal mapping of $f$.

Our algorithm designed for RGO can in fact 
be used for any convex and Lipschitz continuous function $f$. More specifically,
in both settings where the proximal mapping of $f$ exists or not, replacing $\mu$ by $0$, the implementations and results for RGO in this section are also applicable for $f$.

\subsection{Sampling with an optimization oracle}\label{subsec:opt-oracle}

Assume that $f$ has a proximal mapping and let 
\begin{equation}\label{def:x*}
	x^*=\underset{x\in\R^d}\argmin \left\lbrace g^\eta(x):=g(x)+\frac{1}{2\eta}\|x-y\|^2 \right\rbrace.
\end{equation}
Here $y$ is the output of ASF (Algorithm \ref{alg:ASF}) in the previous iteration. Note that solving \eqref{def:x*} is equivalent to invoking one proximal mapping of $f$ since $g(x)=f(x)+\mu\|x-x^0\|^2/2$; the quadratic term can be combined with that in \eqref{def:x*}.

The RGO in each iteration requires sampling from $g^\eta$. Our strategy is rejection sampling with a proper Gaussian proposal centered at $x^*$. This is summarized in Algorithm \ref{alg:RGO}.
The following result is useful in the complexity analysis of Algorithm \ref{alg:RGO}.
It is a special case of a more general result (Lemma \ref{lem:h1h2delta}) we present later and thus the proof is omitted. Throughout, denote
\begin{equation}\label{eq:etamu}
	\eta_\mu:=\eta/(1+\eta\mu).
\end{equation}
\begin{lemma}\label{lem:h1h2}
	Let 
	\[
	h_1:= \frac{1}{2\eta_\mu}\|\cdot-x^*\|^2 + g^\eta(x^*), \quad
	h_2 := \frac{1}{2\eta_\mu}\|\cdot-x^*\|^2 + 2M\|\cdot-x^*\| + g^\eta(x^*).
	\]
	Then, for every $x\in \R^d$, we have $h_1(x) \le g^\eta(x) \le h_2(x)$.
\end{lemma}
\begin{algorithm}[H]
	\caption{Implementation of the RGO with an optimization oracle}
	\label{alg:RGO}
	\begin{algorithmic}
		\STATE 1. Compute $x^*$ as in \eqref{def:x*};
		\STATE 2. Generate  $X\sim \exp(-h_1(x))$;
		\STATE 3. Generate $U\sim {\cal U}[0,1]$;
		\STATE 4. If 
		\[
		U \leq \frac{\exp(-g^\eta(X))}{\exp(-h_1(X))},
		\]
		then accept $\tilde X=X$; otherwise, reject $X$ and go to step 2.
	\end{algorithmic}
\end{algorithm}

The next proposition justifies the correctness and gives the complexity of Algorithm \ref{alg:RGO}. Its proof is postponed to Appendix \ref{sec:proofs}.

\begin{proposition}\label{prop:rejection}
	Assume $f$ is convex and $M$-Lipschitz continuous.
	Let $g=f+\mu\|\cdot-x^0\|^2/2$ and
	\[
	p(x|y) \propto \exp\left(-g(x)-\frac{1}{2\eta}\|x-y\|^2\right)
	\]
	for a fixed $y$, then $\tilde X$ generated by Algorithm \ref{alg:RGO} is such that $\tilde X \sim p(x|y) $.
	If $\eta_\mu \le 1/(16 M^2 d)$, then
	the expected number of iteration in Algorithm \ref{alg:RGO} is at most 2.
\end{proposition}


\subsection{Review of the proximal bundle method}\label{subsec:bundle}

The proximal bundle method \cite{liang2020proximal,liang2021unified} is an efficient algorithm for solving convex non-smooth optimization problems.
In this subsection, we briefly review an approach to solve the subproblem considered in the proximal bundle method, the properties of the solution to the subproblem, and the iteration-complexity for solving the subproblem.

Consider the optimization subproblem (recall \eqref{eq:g})
\begin{equation}\label{eq:subproblem}
	g^\eta_*:= g^\eta (x^*)= \min \left\{ g^\eta(x):= g(x) + \frac{1}{2\eta}\|x-y\|^2: x\in \R^d\right \},
\end{equation}
and we aim at obtaining a $\delta$-solution (i.e., a point $\bar x$ such that $g^\eta(\bar x) - g^\eta_* \le \delta$) to \eqref{eq:subproblem}. 
The algorithm is summarized in Algorithm \ref{alg:PBS}.

We make some remarks about Algorithm \ref{alg:PBS}. First, \eqref{def:Cj} shows the flexibility in the choice of $C_{j+1}$. More specifically, choosing $C_{j+1}=C_j \cup \{x_j\}$ results in the standard cutting-plane model $f_j$ which is underneath $f$, and choosing $C_{j+1}=A_j \cup \{x_j\}$ gives a cutting-plane model $f_j$ with less cuts.
Second, \eqref{def:xj} can be reformulated into a convex quadratic programming with affine constraints and the number of constraints is equal to the cardinality of $C_j$. Since the subproblem \eqref{def:xj} becomes harder to solve as the size of $C_j$ grows, we are in favor of choosing $C_{j+1}$ in \eqref{def:Cj} as lean as possible.

The following lemma contains technical results about Algorithm \ref{alg:PBS} that are useful in the complexity analysis in Subsection \ref{subsec:no opt-oracle}. Its proof is postponed to Appendix \ref{sec:proofs}.

\begin{algorithm}[H]
	\caption{Solving the Proximal Bundle Subproblem \eqref{eq:subproblem}}
	\label{alg:PBS}
	\begin{algorithmic}
		\STATE 0. Let $y$, $\eta>0$ and $\delta>0$ be given, and set $\tx_0=y$, $C_1=\{y\}$ and $j=1$;
		\STATE 1. Update $f_j = \max \left\lbrace  f(x)+\inner{f'(x)}{\cdot-x} : \, x \in C_j\right\rbrace$;
		\STATE 2. Define $g_j:=f_j + \mu\|\cdot-x^0\|^2/2$ and compute
		\begin{equation}
			x_j =\underset{u\in  \R^n}\argmin
			\left\lbrace g_j^\eta(u):= g_j(u) +\frac{1}{2\eta}\|u- y\|^2 \right\rbrace, \label{def:xj}
		\end{equation}
		\[
		\tx_j\in \Argmin\left\lbrace g^\eta(u): u\in \{x_j, \tx_{j-1}\}\right\rbrace;
		\]
		\STATE 3. If $g^\eta(\tx_j) - g_j^\eta(x_j)\le \delta$, then {\bf stop}; else, go to step 4;
		\STATE 4. Choose $C_{j+1}$ such that
		\begin{equation}\label{def:Cj}
			A_{j} \cup \{x_j\} \subset C_{j+1} \subset C_{j}\cup \{x_j\}
		\end{equation}
		where $A_{j}:=\left\lbrace x\in C_{j}: f(x)+\inner{f'(x)}{x_j-x} =f_j(x_j) \right\rbrace$;
		\STATE 5. Set $ j $ \ensuremath{\leftarrow} $ j+1 $ and go to step 1.
	\end{algorithmic}
\end{algorithm}

\begin{lemma}\label{lem:bundle}
	Assume $f$ is convex and $M$-Lipschitz continuous. Let $j$ denote the last iteration index, then the following statements hold:
	\begin{itemize}
		\item[a)] $f_j\le f$, $g_j \le g$ and $g_j^\eta(x_j) + \|x-x_j\|^2/(2\eta_\mu) \le g_j^\eta(x)$ for every $x\in \R^d$;
		\item[b)] $g^\eta(\tx_j) - g_j^\eta(x_j) \le \delta$;
		\item[c)] $\left\|\mu(x_j-x^0) + (x_j-y)/\eta \right\| \le M$;
		\item[d)] $\|x_j - \tx_j\|^2 \le 2\eta_\mu \delta$.
	\end{itemize}
\end{lemma}





The following result states the iteration-complexity bound for Algorithm \ref{alg:PBS} to obtain a $\delta$-solution to the subproblem \eqref{eq:subproblem}. We have omitted the proof since it is relatively technical and beyond the scope of this paper, however, a complete proof can be found in Section 4 of \cite{liang2021unified}.

\begin{proposition}\label{prop:bundle}
	Algorithm \ref{alg:PBS} takes $\tilde {\cal O}(\eta_\mu M^2/\delta + 1)$ iterations to terminate, and each iteration solves an affinely constrained convex quadratic programming problem.
\end{proposition}

\subsection{Sampling without an optimization oracle}\label{subsec:no opt-oracle}

Let $j$ denote the last iteration index of Algorithm \ref{alg:PBS}, i.e., Lemma \ref{lem:bundle} holds with $j$.
Define
\begin{align}
	h_1 &:= \frac{1}{2\eta_\mu}\|\cdot-x_j\|^2 + g^\eta(\tx_j) - \delta, \label{def:tf} \\
	h_2 &:=  \frac{1}{2\eta_\mu}\|\cdot-\tx_j\|^2 + \left(2M + \frac{\sqrt{2\delta}}{\sqrt{\eta_\mu}} \right)\|\cdot-\tx_j\| + g^\eta(\tx_j). \label{def:h}
\end{align}

Algorithms \ref{alg:RGO-bundle} describes the implementation of RGO for $g$ based on Algorithm \ref{alg:PBS} and rejection sampling. It differs from Algorithm \ref{alg:RGO} in that: 1) it uses Algorithm \ref{alg:PBS} to compute an approximate solution to \eqref{eq:subproblem} instead of calling the proximal mapping $f$ as in \eqref{def:x*}.

\begin{algorithm}[H]
	\caption{Implementation of the RGO without an optimization oracle}
	\label{alg:RGO-bundle}
	\begin{algorithmic}
		\STATE 1. Compute $x_j$ and $\tx_j$ as in Algorithm \ref{alg:PBS};
		\STATE 2. Generate  $X\sim \exp(-h_1(x))$;
		\STATE 3. Generate $U\sim {\cal U}[0,1]$;
		\STATE 4. If 
		\[
		U \leq \frac{\exp(-g^\eta(X))}{\exp(-h_1(X))},
		\]
		then accept $\tilde X=X$; otherwise, reject $X$ and go to step 2.
	\end{algorithmic}
\end{algorithm}

The following lemma is a counterpart of Lemma \ref{lem:h1h2} in the context of RGO without an optimization oracle and plays an important role in Proposition \ref{prop:expected}. It reduces to Lemma \ref{lem:h1h2} when $\delta = 0$.

\begin{lemma}\label{lem:h1h2delta}
	Assume $f$ is convex and M-Lipschitz continuous.
	Let $g=f+\mu\|\cdot-x^0\|^2/2$ and $g^\eta$ be as in \eqref{def:x*}.
	Then, for every $x\in \R^d$, we have 
	\begin{equation}\label{ineq:sandwich}
		h_1(x) \le g^\eta(x) \le h_2(x)
	\end{equation}
	where $h_1$ and $h_2$ are as in \eqref{def:tf} and \eqref{def:h}, respectively.
\end{lemma}

\begin{proof}
	Using Lemma \ref{lem:bundle}(a)-(b) and the definition of $g_j^\eta$, we have
	\begin{align*}
		g(\tx_j) - g(x) + \frac{1}{2\eta_\mu}\|x-x_j\|^2 
		& \le g(\tx_j) - g_j(x) + \frac{1}{2\eta_\mu}\|x-x_j\|^2 \\
		& \le g(\tx_j) - g_j^\eta(x_j) + \frac{1}{2\eta}\|x-y\|^2 \\
		& \le \delta - \frac{1}{2\eta}\|\tx_j-y\|^2 + \frac{1}{2\eta}\|x-y\|^2.
	\end{align*}
	The first inequality in \eqref{ineq:sandwich} holds in view of the definition of $h_1$ in \eqref{def:tf}.
	Using the definition of $g^\eta$ in \eqref{def:x*} and the fact that $f$ is $M$-Lipschtz, we have
	\begin{align*}
		&g^\eta(x) - g^\eta(\tx_j) \\
		&= f(x) - f(\tx_j) + \frac{\mu}{2}\|x-x^0\|^2 - \frac{\mu}{2}\|\tx_j-x^0\|^2 + \frac{1}{2\eta}\|x-y\|^2 - \frac{1}{2\eta}\|\tx_j-y\|^2 \\
		& \le M\|x-\tx_j\| + \frac{\mu}{2}\|x-\tx_j\|^2 + \mu\inner{x-\tx_j}{\tx_j-x^0} + \frac{1}{2\eta}\|x-\tx_j\|^2 + \frac{1}{\eta}\inner{x-\tx_j}{\tx_j-y} \\
		&= M\|x-\tx_j\| + \frac{1}{2\eta_\mu}\|x-\tx_j\|^2 + \mu\inner{x-\tx_j}{\tx_j-x_j + x_j -x^0}
		+ \frac{1}{\eta}\inner{x-\tx_j}{\tx_j-x_j + x_j -y}.
	\end{align*}
	The above inequality, the Cauchy-Schwarz inequality and Lemma \ref{lem:bundle}(c)-(d) imply that
	\begin{align*}
		&g^\eta(x) - g^\eta(\tx_j) \\
		&\le M\|x-\tx_j\| + \frac{1}{2\eta_\mu}\|x-\tx_j\|^2 + \frac{1}{\eta_\mu}\|x-\tx_j\| \|\tx_j-x_j\|
		+ \|x-\tx_j\| \left\| \mu(x_j-x^0) + \frac{x_j-y}{\eta} \right\| \\
		&\le M\|x-\tx_j\| + \frac{1}{2\eta_\mu}\|x-\tx_j\|^2 + \frac{\sqrt{2\delta}}{\sqrt{\eta_\mu}}\|x-\tx_j\|  + M\|x-\tx_j\| \\
		&= \left(2M + \frac{\sqrt{2\delta}}{\sqrt{\eta_\mu}} \right) \|x-\tx_j\| + \frac{1}{2\eta_\mu}\|x-\tx_j\|^2.
	\end{align*}
	It follows from the above inequality and the definition of $h_2$ in \eqref{def:h} that the second inequality in \eqref{ineq:sandwich} holds.
\end{proof}

The next proposition is the main result of this subsection and shows that the number of rejections in Algorithm \ref{alg:RGO-bundle} is small in expectation. Hence, the implementation of RGO for $g$ is computationally efficient.
Its proof is postponed to Appendix \ref{sec:proofs}.

\begin{proposition}\label{prop:expected}
	If 
	\begin{equation}\label{ineq:assumption}
		\eta_\mu \le \frac{1}{64 M^2 d}, \quad \delta \le \frac{1}{32d},
	\end{equation}
	then the expected number of iterations in the rejection sampling is at most $3$.
\end{proposition}


\section{Sampling from non-smooth potentials}\label{sec:main}
We now combine our implementation of RGO (Algorithm \ref{alg:RGO-bundle}) and the ASF (Algorithm \ref{alg:ASF}) to sample from log-concave probability densities with non-smooth potentials.
This section contains two subsections. Subsection \ref{subsec:strongly convex} presents
the iteration-complexity bound for sampling from $\exp(-g(x))$ where $g=f+\mu\|\cdot-x^0\|^2/2$. Based on this, Subsection \ref{subsec:convex} provides
the iteration-complexity for sampling from $\exp(-f(x))$ where $f$ is convex and Lipschitz continuous. Apart from being a transition step to our final result for sampling without smoothness, the results in Subsection \ref{subsec:strongly convex} provide an efficient method to sample from composite potentials of the form $f+\mu\|\cdot-x^0\|^2/2$ and may be of independent interest.

\subsection{Total complexity for strongly convex potential}\label{subsec:strongly convex}

Using the efficient implementation of RGO for $g$ developed in Section \ref{sec:RGO} and the alternating sampling framework Algorithm \ref{alg:ASF}, we are now able to sample from $\exp(-g(x))$ and establish the complexity for this sampling task.
The following theorem states the iteration-complexity bound for Algorithm \ref{alg:ASF} using Algorithm \ref{alg:RGO-bundle} as the RGO to sample from $\exp(-g(x))$. Note that this iteration-complexity bound is poly-logarithmic in the precision $\varepsilon$ in terms of total variation.
Its proof is postponed to Appendix \ref{sec:proofs}.

\begin{theorem}\label{thm:g}
	Let $x^0\in\R^d$, $\varepsilon>0$, $\delta>0$, $M>0$, $\mu>0$ and $\eta>0$ satisfying
	\begin{equation}\label{ineq:eta}
		\frac{\delta}{M^2}\le \eta \le \min\left\lbrace\frac{1}{64M^2 d}, \frac{1}{\mu} \right \rbrace
	\end{equation}
	be given.
	Let $\pi$ be a distribution on $\R^d$ satisfying $\pi(x) \propto \exp(-g(x))=\exp(-f(x) - \mu\|x-x^0\|^2/2)$ where $f$ is convex and $M$-Lipschitz continuous on $\R^d$. 
	Consider Algorithm \ref{alg:ASF} using Algorithm \ref{alg:RGO-bundle} as an RGO for step 1, initialized at the minimizer of $g$,
	then the iteration-complexity bound for obtaining $\varepsilon$ total tolerance to $\pi$ in terms of total variation is
	\begin{equation}\label{cmplx:mu}
		\tilde {\cal O}\left(
		\frac{M^2}{\mu \delta} \log\left( \frac{d}{\eta \mu \varepsilon}\right) + 1  \right),
	\end{equation}
	and each iteration queries one subgradient oracle of $f$ and solves a quadratic programming problem. 
	Moreover, the number of Gaussian distribution sampling queries in Algorithm \ref{alg:ASF} is
	\begin{equation}\label{cmplx:outer}
		\Theta \left(\frac{1}{\eta \mu} \log\left( \frac{d}{\eta \mu \varepsilon}\right) + 1 \right).
	\end{equation}
\end{theorem}


It is worth noting that if the proximal mapping of $f$ exists, then the implementation of RGO only requires one call to the proximal mapping and a number of rejection sampling. As a result, the total complexity for Algorithm \ref{alg:ASF} is the same as the complexity in Theorem \ref{thm:outer}, i.e., $\Theta\left(\frac{1}{\eta \mu} \log \frac{d}{\eta \mu \varepsilon}\right)$.

\subsection{Total complexity for convex potential}\label{subsec:convex}

This subsection studies the main problem of this paper, i.e., sampling from $\exp(-f(x))$.
Building upon Theorem \ref{thm:g} for sampling from $\exp(-g(x))=\exp(-f(x) - \mu\|x-x^0\|^2/2)$ and a proper choice of $\mu$, the following theorem establishes the iteration-complexity bound for Algorithm \ref{alg:ASF} to sample from $\exp(-f(x))$.
Its proof is postponed to Appendix \ref{sec:proofs}.

\begin{theorem}\label{thm:f}
	Let $\pi$ be a distribution on $\R^d$ satisfying $\pi(x) \propto \exp(-f(x))$ where $f$ is convex and $M$-Lipschitz continuous on $\R^d$. 
	Let $x^0\in\R^d$ and $\varepsilon>0$ be given and 
	\begin{equation}\label{eq:mu}
		\mu = \frac{\varepsilon}{\sqrt{2}\left(\sqrt{{\cal M}_4} + \|x^0-x_\text{min}\|^2 \right)}
	\end{equation}
	where ${\cal M}_4=\int_{x \in \R^d}\|x-x_\text{min}\|^4 d\pi(x)$ and $x_\text{min} \in \Argmin \{f(x): x\in \R^d\}$.
	Choose $\delta>0$ and $\eta>0$ such that \eqref{ineq:eta} holds and
	consider Algorithm \ref{alg:ASF} using Algorithm \ref{alg:RGO-bundle} as an RGO for step 1, applied to $g=f+\mu\|\cdot-x^0\|^2/2$, and initialized at 
	the minimizer of $g$.
	Then, the iteration-complexity bound for obtaining $\varepsilon$ total tolerance to $\pi$ is
	\begin{equation}\label{cmplx:mu1}
		\tilde {\cal O}\left(
		\frac{M^2 \left(\sqrt{{\cal M}_4} + \|x^0-x_\text{min}\|^2 \right) }{\varepsilon \delta} \log\left( \frac{d\left(\sqrt{{\cal M}_4} + \|x^0-x_\text{min}\|^2 \right)}{\eta  \varepsilon^2}\right) + 1  \right).
	\end{equation}
\end{theorem}


Finally, we remark that \eqref{ineq:eta} implies that $\delta \le 1/(64 d)$, and if we choose $\delta = C d^{-1}$ for some universal constant $C>0$, then the total complexity \eqref{cmplx:mu1} for sampling from non-smooth potentials becomes $\tilde {\cal O}(M^2d\sqrt{{\cal M}_4}\varepsilon^{-1})$.

\begin{remark}
	The strategy we use to sample from a non-smooth potential $f$ by considering a regularized one $g=f+\mu\|\cdot-x^0\|^2/2$ first is not the only option. The reason we do so is that the complexity bound for ASF in \cite{lee2020structured} requires the potential to be strongly convex. This convergence result for ASF can, however, be extended. Following a similar argument as in \cite{lee2020structured}, in particular Proposition 2 and Lemma 2, one can establish the complexity bound (with respect to total variation) ${\cal O}(\frac{1}{\eta\psi^2}\log(\beta/\varepsilon))$ for ASF with convex potential where $\psi$ is the isoperimetry constant of the target distribution and $\beta$ is a warm start constant. Combining this with our RGO implementation (Algorithm \ref{alg:RGO-bundle}) yields a method to sample from non-smooth potential $f$ with complexity ${\cal O}(\frac{M^2 d}{\psi^2}\log(\beta/\varepsilon))$. This is better than $ {\cal O} (d^{5/2} \log(\beta/\varepsilon) )$ in \cite{lee2017eldan}, even in the high accuracy region, if $M$ scales slower than $\sqrt{d}$, which is typical if sparsity exists. 
\end{remark}

\section{Sampling from smooth potentials}\label{sec:smooth}
Surprisingly, the exactly same algorithm (Algorithm \ref{alg:RGO-bundle}) we developed is applicable to smooth potentials.
In this section, we extend the proximal sampling algorithm developed for non-smooth sampling to its smooth counterpart, and establish iteration-complexity results.
Subsection \ref{subsec:RGO-smooth} analyzes the RGO (Algorithm \ref{alg:RGO-bundle}) without an optimization oracle for smooth potentials.
Subsection \ref{subsec:total-smooth} provides the total complexity results for sampling from log-concave probability densities with non-smooth potentials; it is the counterpart of Section \ref{sec:main} for the smooth case.

We assume in this section that $f$ is convex and $L$-smooth, i.e., for every $u,v \in \R^d$,
\begin{equation}\label{ineq:smooth}
	f(u) - f(v) - \inner{\nabla f(v)}{u-v} \le \frac L2 \|u-v\|^2.
\end{equation}
Considering Algorithm \ref{alg:ASF} using Algorithm \ref{alg:RGO-bundle} as an RGO for step 1, the goal of this section is to show that this algorithm, which is originally designed for non-smooth sampling, is also able to sample from smooth potentials.

\subsection{Analysis of RGO without an optimization oracle}\label{subsec:RGO-smooth}

For simplicity, we focus on the case where we do not have an optimization oracle for $f$. The analysis of RGO with an optimization oracle can be obtained by taking $\delta=0$. In such a case, we have $x_j=\tx_j=x^*$ (Recall $j$ denotes the last iteration index of Algorithm \ref{alg:PBS}).






The following lemma is a version of Lemma \ref{lem:h1h2delta} in the smooth case. Its proof is postponed to Appendix \ref{sec:proofs}.

\begin{lemma}\label{lem:h1h2-smooth-2}
	Assume $f$ is convex and $L$-smooth.
	Let $g=f+\mu\|\cdot-x^0\|^2/2$ and $g^\eta$ be as in \eqref{def:x*}.
	Let $h_1$ be as in \eqref{def:tf}, and define 
	\begin{equation}\label{def:h2}
		h_2 := \frac{1}{2\eta_{\mu,L}}\|\cdot-x^*\|^2 + g^\eta(x^*)
	\end{equation}
	where
	\begin{equation}\label{def:eta-mu-L}
		\eta_{\mu,L}:=\frac{\eta}{1+\eta\mu+\eta L}.
	\end{equation}
	Then, for every $x\in \R^d$, we have $h_1(x) \le g^\eta(x) \le h_2(x)$.
\end{lemma}

The next proposition gives the number of rejections in Algorithm \ref{alg:RGO-bundle} and is the counterpart of Proposition \ref{prop:expected} in the context of $f$ being $L$-smooth. Its proof is postponed to Appendix \ref{sec:proofs}.

\begin{proposition}\label{prop:rejection-smooth}
	If $\eta_\mu \le 1/(L d)$,
	then the expected number of iterations in the rejection sampling is at most $\exp(1/2+ \delta)$.
\end{proposition}


The following result provides the iteration-complexity for obtaining a $\delta$-solution by Algorithm \ref{alg:PBS} in the smooth setting.
We omit the proof since it is relatively technical, however, a complete proof can be found in Section 4 of \cite{liang2021unified}.
Note that the $\tilde {\cal O}$ notation hides $\log(\delta^{-1})$.

\begin{proposition}\label{prop:bundle-smooth}
	Algorithm \ref{alg:PBS} takes $\tilde {\cal O}(\eta_\mu L + 1)$ iterations to terminate, and each iteration solves an affinely constrained convex quadratic programming problem.
\end{proposition}

\subsection{Total complexity}\label{subsec:total-smooth}

We establish two total iteration-complexity bounds to sample from log-concave probability densities with smooth potentials. 

The first theorem states the iteration-complexity bound for sampling from $\exp(-g(x))$ where $g=f+\mu\|\cdot-x^0\|^2/2$, and $f$ is convex and $L$-smooth. Its proof is postponed to Appendix \ref{sec:proofs}.

\begin{theorem}\label{thm:g-smooth}
	Let $x^0\in\R^d$, $\varepsilon>0$, $L>0$, $\mu>0$ and $\eta>0$ satisfying $\eta \le 1/(L d)$
	be given.
	Let $\pi$ be a distribution on $\R^d$ satisfying $\pi(x) \propto \exp(-g(x))=\exp(-f(x) - \mu\|x-x^0\|^2/2)$ where $f$ is convex and $L$-smooth on $\R^d$. 
	Consider Algorithm \ref{alg:ASF} using Algorithm \ref{alg:RGO-bundle} as an RGO for step 1, initialized at the minimizer of $g$,
	then the iteration-complexity bound for obtaining $\varepsilon$ total tolerance to $\pi$ in terms of total variation is
	\begin{equation}\label{cmplx:mu-smooth}
		\tilde {\cal O}\left(
		\frac{Ld}{\mu} \log\left( \frac{L d^2}{\mu \varepsilon}\right) + 1  \right),
	\end{equation}
	and each iteration queries a gradient $\nabla f$ and solves a quadratic programming problem. 
\end{theorem}


The next theorem is the main result of this section, which studies the total iteration-complexity for sampling from $\exp(-f(x))$ where $f$ is convex and $L$-smooth.
We omit the proof since the theorem can be similarly proved by following the proof of Theorem \ref{thm:f}.

\begin{theorem}\label{thm:f-smooth}
	Let $\pi$ be a distribution on $\R^d$ satisfying $\pi(x) \propto \exp(-f(x))$ where $f$ is convex and $L$-smooth on $\R^d$. 
	Let $x^0\in\R^d$ and $\varepsilon>0$ be given and $\mu$ be as in \eqref{eq:mu}.
	Choose $\eta\le 1/(Ld)$ and
	consider Algorithm \ref{alg:ASF} using Algorithm \ref{alg:RGO-bundle} as an RGO for step 1, applied to $g=f+\mu\|\cdot-x^0\|^2/2$, and initialized at 
	the minimizer of $g$.
	Then, the iteration-complexity bound for obtaining $\varepsilon$ total tolerance to $\pi$ is
	\[
	\tilde {\cal O}\left(
	\frac{L d \left(\sqrt{{\cal M}_4} + \|x^0-x_\text{min}\|^2 \right) }{\varepsilon} \log\left( \frac{L d^2 \left(\sqrt{{\cal M}_4} + \|x^0-x_\text{min}\|^2 \right)}{\varepsilon^2}\right) + 1  \right).
	\]
\end{theorem}



\section{Conclusion}\label{sec:conclusion}

This paper presents an algorithm based on the alternating sampling framework for sampling from non-smooth potentials and establishes a complexity bound $\tilde {\cal O}(d\varepsilon^{-1})$ to obtain $\varepsilon$ total variation distance to the target density.
Moreover, the algorithm is also applicable to sample from smooth potentials and has a complexity bound $\tilde {\cal O}(d\varepsilon^{-1})$.
The key contribution of this paper is a computationally efficient implementation of RGO for any convex (either smooth or non-smooth) function.
One direct extension of the paper is to apply the proposed algorithm to sample from semi-smooth densities, which include smooth and non-smooth densities as two extreme cases.
Another possible extension of our analysis in this paper is to consider sampling from composite densities proportional to $\exp(-f(x)-h(x))$ where $f$ is convex and smooth, and $h$ is convex and semi-smooth.

\section*{Acknowledgement}
This work was supported by NSF under grant 1942523 and 2008513.

\bibliographystyle{plain}
\bibliography{ref}

\appendix

\section{Technical results}

\begin{lemma}\label{lem:Gaussian}
	Useful Gaussian integrals:
	\begin{itemize}
		\item[a)] for any $\lam>0$,
		\[
		\int_{\R^d} \exp\left(-\frac{1}{2\lam}\|x\|^2\right) dx = (2\pi \lam)^{d/2};
		\]
		\item[b)] for any $c>0$ and $n\ge 1$,
		\begin{equation*}
			\int_0^\infty \exp(-cx^2) x^n dx = \left\{\begin{array}{ll}
				\frac{(n-1)!!}{2^{n/2+1} c^{n/2}} \sqrt{\frac{\pi}{c}}, & \text { for } n \text{ even}, \\[0.15cm]
				\frac{\left(\frac{n-1}{2}\right)!}{2 c^{(n+1)/2}}, & \text { for } n \text{ odd}.
			\end{array}\right.
		\end{equation*}
	\end{itemize}
\end{lemma}

\begin{lemma}\label{lem:Gamma}
	Facts about the Gamma function $\Gamma$:
	for every $k\ge 1$,
	\begin{equation}\label{eq:Gamma}
		\Gamma(k)=(k-1)!, \quad \Gamma\left(k+\frac12\right) = \left(k-\frac12\right)\left(k-\frac32\right) \cdots \frac12 \sqrt{\pi},
	\end{equation}
	and
	\begin{equation}\label{ineq:double}
		\sqrt{k}< \frac{\Gamma(k+1)}{\Gamma(k+\frac12)} < \sqrt{k+\frac12}.
	\end{equation}
\end{lemma}

\begin{proof}
	Identities in \eqref{eq:Gamma} are well-known and hence their proofs are omitted.
	We give an elementary proof of \eqref{ineq:double}.
	Let
	\[
	I_k= \int_0^{\pi/2} \sin^k x  dx, \quad \forall k\ge 0,
	\]
	then integration by parts gives the recursive formula
	\begin{equation}\label{eq:Ik}
		I_k = \frac{k-1}{k} I_{k-2}, \quad \forall k \ge 2.
	\end{equation}
	Applying the above identity recursively and using the facts that $I_0=\pi/2$ and $I_1=1$, we have for every $k\ge 1$,
	\[
	I_{2k}= \frac{2k-1}{2k} \frac{2k-3}{2k-2} \cdots \frac{1}{2} \frac{\pi}{2}, \quad I_{2k+1}= \frac{2k}{2k+1} \frac{2k-2}{2k-1} \cdots \frac{2}{3},
	\]
	and hence,
	\begin{equation}\label{eq:ratio}
		\frac{I_{2k}}{I_{2k+1}} = \frac{2k+1}{2} \frac{((2k)!)^2}{4^{2k} (k!)^4} \pi.
	\end{equation}
	It follows from the fact that $\sin x < 1$ for $x\in (0,\pi/2)$ that $I_{k+1}< I_k$ for every $k\ge 0$.
	This observation together with \eqref{eq:Ik} implies that
	\[
	1< \frac{I_{2k}}{I_{2k+1}} < \frac{I_{2k-1}}{I_{2k+1}} = \frac{2k+1}{2k}.
	\]
	Using the above inequality and \eqref{eq:ratio}, we have
	\begin{equation}\label{ineq:I}
		\frac{1}{\sqrt{k+\frac12}}< \frac{(2k)!}{4^k (k!)^2} \sqrt{\pi} = \frac{1}{4^k} { 2k \choose k } \sqrt{\pi} < \frac{1}{\sqrt{k}}.
	\end{equation}
	Finally, it follows from \eqref{eq:Gamma} that
	\[
	\frac{\Gamma\left(k+\frac12\right)}{\Gamma(k+1)} = \frac{1}{4^k} { 2k \choose k } \sqrt{\pi},
	\]
	which together with \eqref{ineq:I} implies \eqref{ineq:double}.
\end{proof}

\begin{proposition}\label{lem:key}
	For $a\ge0$ and $d\ge 1$, if $\lam \le 1/(16 a^2 d)$, then
	\begin{equation}\label{ineq:int}
		\int_{\R^d} \exp\left(-\frac{1}{2\lam}\|x\|^2 - 2a\|x\|\right) d x
		\ge \frac{(2\pi\lam)^{d/2}}2.
	\end{equation}
\end{proposition}

\begin{proof}
	Let $r=\|x\|$ and note that
	\[
	dx= r^{d-1} d S^{d-1}dr
	\]
	where $ S^{d-1} $ is the surface area of the $ (d-1) $-dimensional unit sphere.
	Define
	\begin{equation}\label{def:F}
		F_{d,\lam}(a):= \int_0^\infty \exp\left( -\frac{1}{2\lam}r^2 - 2a r \right) r^{d} d r,
	\end{equation}
	then we have
	\begin{align}
		\int \exp\left(-\frac{1}{2\lam}\|x\|^2 - 2a\|x\|\right) d x &= \int \exp\left(-\frac{1}{2\lam}r^2 - 2a r\right) r^{d-1} d r d S^{d-1} \nn \\
		&= \frac{2 \pi^{d/2}}{\Gamma\left( \frac d2 \right) } F_{d-1,\lam}(a).  \label{eq:int}
	\end{align}
	In the above, we have used the fact that the total surface area of a $ (d-1) $-dimensional unit sphere is $2 \pi^{d/2}/\Gamma\left( \frac d2 \right)$.
	It follows from the definition of $F_{d,\lam}$ in \eqref{def:F} that
	\[
	\frac{d F_{d-1,\lam}(a)}{d a}=\int_0^\infty \exp\left( -\frac{1}{2\lam}r^2 - 2 a r \right) (-2r) r^{d-1} d r 
	=-2 F_{d,\lam}(a) \ge -2F_{d,\lam}(0),
	\]
	and hence that
	\begin{equation}\label{ineq:F}
		F_{d-1,\lam}(a) \ge F_{d-1,\lam}(0) -2 a F_{d,\lam}(0).
	\end{equation}
	We now consider two cases: $d=2k+1$ for $k\ge 0$ and $d=2k$ for $k\ge 1$.
	
	\noindent
	{\bf Case 1}: $d=2k+1$. It follows from the definition of $F_{d,\lam}$ in \eqref{def:F} and Lemma \ref{lem:Gaussian}(b) that
	\begin{align}
		&F_{2k,\lam}(0)=\int_0^\infty \exp\left( -\frac{1}{2\lam}r^2 \right) r^{2k} d r
		= \frac{(2k-1)!!\lam^k}{2}\sqrt{2\lam \pi}, \label{eq:2k} \\
		& F_{2k+1,\lam}(0)=\int_0^\infty \exp\left( -\frac{1}{2\lam}r^2 \right) r^{2k+1} d r
		= \frac{k!(2\lam)^{k+1}}{2}. \nn
	\end{align}
	Using the above two identities and \eqref{ineq:F} with $d=2k+1$, we have
	\begin{align*}
		F_{2k,\lam}(a) &\ge F_{2k,\lam}(0) -2 a F_{2k+1,\lam}(0)\\
		&=\frac{(2k-1)!!\lam^k}{2}\sqrt{2\lam \pi} - a k!(2\lam)^{k+1} \\
		&= \lam^k (2k-1)!! \left( \sqrt{\frac{\lam \pi}2} - 2a \lam \frac{(2k)!!}{(2k-1)!!} \right).
	\end{align*}
	The above inequality, the fact that
	\[
	\frac{(2k)!!}{(2k-1)!!} = \frac{\Pi_{i=1}^k 2i}{\Pi_{i=1}^k (2i-1)} = \frac{\Pi_{i=1}^k i}{\Pi_{i=1}^k (i-\frac12)} = \frac{\Gamma(k+1)}{\Gamma(k+\frac12)} \sqrt{\pi}
	\]
	and \eqref{ineq:double} imply that
	\begin{align*}
		F_{2k,\lam}(a) &\ge \lam^k (2k-1)!! \left( \sqrt{\frac{\lam \pi}2} - 2a \lam \sqrt{\left(k+\frac12\right) \pi} \right) \\
		&= \lam^k (2k-1)!! \left( \sqrt{\frac{\lam \pi}2} - a \lam \sqrt{2 d \pi} \right) \\
		&\ge \lam^k (2k-1)!!\frac12 \sqrt{\frac{\lam \pi}2}\\
		&= (2\lam)^k \Gamma\left(k+\frac12\right) \frac{\sqrt{\lam}}{2\sqrt{2}}
		= \frac{(2\lam)^{d/2} \Gamma(\frac d2) }4 
	\end{align*}
	where the second inequality is due to the assumption that $\lam \le 1/(16 a^2 d)$, and the first identity is due to the fact that $d=2k+1$.
	Using the above inequality, \eqref{eq:int} and the fact that $d=2k+1$, we conclude that \eqref{ineq:int} holds for case 1.
	
	\noindent
	{\bf Case 2}: $d=2k$. 
	It follows from the definition of $F_{d,\lam}$ in \eqref{def:F} and Lemma \ref{lem:Gaussian}(b) that
	\[
	F_{2k-1,\lam}(0)=\int_0^\infty \exp\left( -\frac{1}{2\lam}r^2 \right) r^{2k-1} d r
	= \frac{(k-1)!(2\lam)^k}{2} = (2k-2)!!\lam^k.
	\]
	Using the above identity, \eqref{eq:2k}, and \eqref{ineq:F} with $d=2k$, we have
	\begin{align*}
		F_{2k-1,\lam}(a) &\ge F_{2k-1,\lam}(0) -2 a F_{2k,\lam}(0)\\
		&= (2k-2)!!\lam^k \left( 1 - a \sqrt{2\lam \pi} \frac{(2k-1)!!}{(2k-2)!!} \right) \\
		&= (2k-2)!!\lam^k \left( 1 - 2 a k \sqrt{2\lam} \frac{\Gamma(k+\frac12)}{\Gamma(k+1)} \right) \\
		&\ge (2k-2)!!\lam^k \left( 1 - 2a \sqrt{2\lam k} \right) 
	\end{align*}
	where the second identity is due to \eqref{eq:Gamma}, and the second inequality is due to \eqref{ineq:double}.
	It follows from the above inequality, the fact that $d=2k$ and the assumption that $\lam \le 1/(16 a^2 d)$ that
	\begin{align*}
		F_{d-1,\lam}(a) & \ge (2k-2)!!\lam^k \left( 1 - 2a \sqrt{\lam d} \right) 
		\ge \frac{(2k-2)!!\lam^k}{2}\\
		&= \frac{(k-1)!!(2\lam)^k}{4} = \frac{\Gamma(k)(2\lam)^k}{4} = \frac{\Gamma(\frac d2)(2\lam)^{d/2}}{4}
	\end{align*}
	where the second identity is due to \eqref{eq:Gamma}.
	Using the above inequality, \eqref{eq:int} and the fact that $d=2k$, we conclude that \eqref{ineq:int} holds for case 2.
\end{proof}

\section{Missing proofs}\label{sec:proofs}

\noindent
{\bf Proof of Proposition \ref{prop:rejection}:}
It is a well-known result for rejection sampling that $\tilde X \sim p(x)$ and the probability that $\tilde X$ is accepted is
\[
\mathbb{P}\left(U \leq \frac{\exp(-g^\eta(X))}{\exp(-h_1(X))}\right) 
=\frac{\int \exp(-g^\eta(x)) d x}{\int \exp(-h_1(x)) d x}.
\]
Using the assumption that $\eta_\mu \le 1/(16 M^2 d)$, Lemma \ref{lem:h1h2} and Proposition \ref{lem:key} that
\[
\int_{\R^d} \exp(-g^\eta(x)) d x \ge
\int_{\R^d} \exp(-h_2(x)) d x \ge \exp(-g^\eta(x^*)) \frac{(2\pi\eta_\mu)^{d/2}}2.
\]
Moreover, it follows from the definition of $h_1$ and Lemma \ref{lem:Gaussian}(a) that
\[
\int_{\R^d} \exp(-h_1(x)) d x = \exp(-g^\eta(x^*)) (2\pi\eta_\mu)^{d/2}.
\]
The above three relations immediately imply that 
\[
\mathbb{P}\left(U \leq \frac{\exp(-g^\eta(X))}{\exp(-h_1(X))}\right) \ge \frac12,
\]
and hence that the expected number of iterations is bounded above by 2. 
\QEDA

\vgap

\noindent
{\bf Proof of Lemma \ref{lem:bundle}:}
a) The first two inequalities directly follow from the convexity of $f$, and the definitions of $g$ and $g_j$. The third inequality follows from \eqref{def:xj}.

b) This statement immediately follows from step 3 of Algorithm \ref{alg:PBS}.

c) It follows from the optimality condition of \eqref{def:xj} and the definition of $g_j$ that
\[
-\mu(x_j-x^0) - \frac{x_j-y}{\eta} \in \partial f_j(x_j).
\]
This inclusion and the fact that $\|f'(x)\|\le M$ imply that c) holds.

d) 
The last inequality in (a) with $x=\tx_j$ and (b) imply this statement.
\QEDA

\vgap

\noindent
{\bf Proof of Proposition \ref{prop:expected}:}
We first observe that the assumption \eqref{ineq:assumption} implies that
\[
2\sqrt{\eta_\mu} M + \sqrt{2\delta} \le \frac{1}{2\sqrt{d}},
\]
which satisfies the assumption in Proposition \ref{lem:key} with
\[
\lam = \eta_\mu, \quad a = M + \frac{\sqrt{\delta}}{\sqrt{2\eta_\mu}}.
\]
Using the definition of $h_2$ in \eqref{def:h} and Lemma \ref{lem:key},
we have
\[
\int_{\R^d} \exp(-h_2(x)) d x \ge \frac12 \exp(-g^\eta(\tx_j)) (2\pi \eta_\mu)^{d/2}.
\]
It follows from the definition of $h_1$ in \eqref{def:tf} and Lemma \ref{lem:Gaussian}(a) with $\lam=\eta_\mu$ that
\[
\int \exp(-h_1(x)) dx 
= \exp\left( -g^\eta(\tx_j) + \delta \right) (2\pi \eta_\mu)^{d/2}.
\]
We conclude that
\begin{align*}
	\mathbb{P}\left(U \leq \frac{\exp(-g^\eta(X))}{\exp(-h_1(X))}\right) &=\frac{\int \exp(-g^\eta(x)) d x}{\int \exp(-h_1(x)) d x} \\
	&\ge \frac{\int \exp(-h_2(x)) d x}{\exp\left( -g^\eta(\tx_j)+\delta \right) (2\pi \eta_\mu)^{d/2}} \ge \frac12 \exp(-\delta),
\end{align*}
and the expected number of the iterations is
\[
\frac{1}{\mathbb{P}\left(U \leq \frac{\exp(-g^\eta(X))}{\exp(-f(X))}\right)}
\le 2\exp(\delta) \le 2(1+2\delta) \le 2\left(1+\frac{1}{16d}\right) \le 3
\]
where the last two inequalities are due to the second inequality in \eqref{ineq:assumption}.
\QEDA

\vgap

\noindent
{\bf Proof of Theorem \ref{thm:g}:}
It follows from Theorem \ref{thm:outer} that the iteration-complexity for Algorithm \ref{alg:ASF} to obtain  $\varepsilon$ total tolerance to $\pi$ is \eqref{cmplx:outer}, which together with Proposition \ref{prop:bundle} implies that the total iteration-complexity is 
\begin{equation}\label{cmplx:total}
	\tilde {\cal O}\left( \left[ \frac{\eta_\mu M^2}\delta + 1\right] 
	\left[\frac{1}{\eta \mu} \log\left( \frac{d}{\eta \mu \varepsilon}\right) + 1 \right] \right).
\end{equation}
Let
\[
a = \frac{\eta_\mu M^2}{\delta}, \quad b = \frac{1}{\eta\mu}.
\]
In view of the above definitions of $a$ and $b$, the total iteration-complexity \eqref{cmplx:total} becomes $\tilde {\cal O}((a+1)(b+1))$.
It is easy to see from \eqref{ineq:eta} that $a \ge 1/2$ and $b \ge 1$, and hence $\tilde {\cal O}((a+1)(b+1))$ is equal to $\tilde {\cal O}(ab + 1)$. Since $\eta_\mu \le \eta$, the total iteration-complexity $\tilde {\cal O}(ab + 1)$ is \eqref{cmplx:mu}.
Moreover, it follows from \eqref{ineq:eta} that \eqref{ineq:assumption} is satisfied, and hence that Proposition \ref{prop:expected} holds. This conclusion together with the iteration-complexity \eqref{cmplx:outer} for Algorithm \ref{alg:ASF} implies the last conclusion of the theorem.
\QEDA

\vgap

\noindent
{\bf Proof of Theorem \ref{thm:f}:}
Let $\rho$ denote the distribution of the points generated by Algorithm \ref{alg:ASF} using Algorithm \ref{alg:RGO-bundle} as an RGO, and let $\hat \pi$ denote the distribution proportional to $\exp(-g(x))$
Following the proof of Corollary 4.1 of \cite{chatterji2020langevin}, we similarly have
\[
\|\rho- \pi\|_{\text{TV}} \le \|\rho-\hat \pi\|_{\text{TV}} + \|\hat \pi - \pi\|_{\text{TV}}
\]
and
\begin{align*}
	\|\hat \pi &- \pi\|_{\text{TV}}  \le \frac12 \left( \int_{\R^d} [f(x)-g(x)]^2 d \pi(x) \right)^{1/2} = \frac12 \left( \int_{\R^d} \left(\frac{\mu}{2}\|x-x^0\|^2 \right) ^2 d \pi(x) \right)^{1/2}\\
	&\le \frac\mu2 \left( \int_{\R^d} \left(\|x-x_\text{min}\|^2 + \|x_\text{min}-x^0\|^2 \right)^2 d \pi(x) \right)^{1/2} \\
	&\le \frac{\mu}2 \left( \int_{\R^d} \left(2\|x-x_\text{min}\|^4 + 2\|x_\text{min}-x^0\|^4 \right) d \pi(x) \right)^{1/2} \\
	&= \frac{\sqrt{2}\mu}2 \left( {\cal M}_4 + \|x_\text{min}-x^0\|^4 \right)^{1/2}
	\le \frac{\sqrt{2}\mu}{2} \left(\sqrt{{\cal M}_4} + \|x_0-x_\text{min}\|^2 \right) = \frac{\varepsilon}{2}
\end{align*}
where the last identity is due to the definition of $\mu$ in \eqref{eq:mu}.
Hence, it suffices to derive the iteration-complexity bound for Algorithm \ref{alg:ASF} to obtain $\|\rho-\hat \pi\|_{\text{TV}}\le \varepsilon/2$, which is \eqref{cmplx:mu1} in view of Theorem \ref{thm:g} with $\mu$ as in \eqref{eq:mu}.
\QEDA

\vgap

\noindent
{\bf Proof of Lemma \ref{lem:h1h2-smooth-2}:}
The first inequality $h_1(x) \le g^\eta(x)$ immediately follows Lemma \ref{lem:h1h2delta}.
We observe that the optimality condition of \eqref{eq:subproblem} is
\begin{equation}\label{observation}
	\mu(x^*-x^0) + \frac{1}{\eta} (x^*-y) = -\nabla f(x^*).
\end{equation}
It follows from the definition of $g^\eta$ in \eqref{def:x*} and the observation \eqref{observation} that
\begin{align}
	&g^\eta(x) - g^\eta(x^*) \nn \\
	= & f(x) - f(x^*) + \frac{\mu}{2}\|x-x^0\|^2 - \frac{\mu}{2}\|x^*-x^0\|^2 + \frac{1}{2\eta}\|x-y\|^2 - \frac{1}{2\eta}\|x^*-y\|^2 \nn \\
	= & f(x) - f(x^*) + \frac{\mu}{2}\|x-x^*\|^2 + \mu \inner{x-x^*}{x^*-x^0} + \frac{1}{2\eta}\|x-x^*\|^2 + \frac{1}{\eta}\inner{x-x^*}{x^*-y} \nn \\
	= & f(x) - f(x^*) + \inner{\mu(x^*-x^0) + \frac{1}{\eta} (x^*-y) }{x-x^*} + \frac{1}{2\eta_\mu}\|x-x^*\|^2 \nn \\
	= & f(x) - f(x^*) - \inner{\nabla f(x^*)}{x-x^*} + \frac{1}{2\eta_\mu}\|x-x^*\|^2.
	\label{eq:equal}
\end{align}
The above inequality and \eqref{ineq:smooth} with $(u,v)=(x,x^*)$ imply that
\[
g^\eta(x) - g^\eta(x^*) \le \frac{L}{2}\|x-x^*\|^{2} + \frac{1}{2\eta_\mu}\|x-x^*\|^2.
\]
Using the above inequality and the definitions of $\eta_{\mu,L}$ and $h_2$ in \eqref{def:eta-mu-L} and \eqref{def:h2}, we conclude that the second inequality $g^\eta(x)\le h_2(x)$ holds.
\QEDA

\vgap

\noindent
{\bf Proof of Proposition \ref{prop:rejection-smooth}:}
Using the definition of $h_2$ in \eqref{def:h2} and Lemma \ref{lem:Gaussian} with $\lam=\eta_{\mu,L}$, we have
\begin{align*}
	\int_{\R^d} \exp(-h_2(x)) d x &= \exp(-g^\eta(x^*)) \int_{\R^d} \exp\left(-\frac{1}{2\eta_{\mu,L}}\|x-x^*\|^2\right) dx \\
	&= \exp(-g^\eta(x^*)) (2\pi \eta_{\mu,L})^{d/2}.
\end{align*}
The above identity, the fact that $g^\eta(\tx_j) \ge g^\eta(x^*)$, and the definitions of $\eta_\mu$ and $\eta_{\mu,L}$ imply that
\begin{align*}
	\mathbb{P}\left(U \leq \frac{\exp(-g^\eta(X))}{\exp(-h_1(X))}\right) 
	&\ge \frac{\int_{\R^d} \exp(-h_2(x)) d x}{\exp\left( -g^\eta(\tx_j)+\delta \right) (2\pi \eta_\mu)^{d/2}} \\
	& \ge \exp(g^\eta(\tx_j) - g^\eta(x^*) - \delta) \left( \frac{\eta_{\mu,L}}{\eta_\mu}\right)^{d/2}
	\ge \exp(-\delta) \left( \frac{1}{1+\eta_\mu L}\right)^{d/2}.
\end{align*}
The above inequality and the assumption that $\eta_\mu\le 1/(Ld)$ imply that the expected number of the iterations is
\[
\frac{1}{\mathbb{P}\left(U \leq \frac{\exp(-g^\eta(X))}{\exp(-h_1(X))}\right) }
\le \exp(\delta) (1+\eta_\mu L)^{d/2} \le \exp(\delta) \left(1+\frac{1}{d}\right)^{d/2} \le \exp(1/2+ \delta),
\]
where the last inequality is due to the fact that $(1+1/d)^d\le e$.
\QEDA

\vgap

\noindent
{\bf Proof of Theorem \ref{thm:g-smooth}:}
It follows from Theorem \ref{thm:outer} that the iteration-complexity (i.e., number of calls to RGO) for Algorithm \ref{alg:ASF} to obtain  $\varepsilon$ total tolerance to $\pi$ is \eqref{cmplx:outer}. Using \eqref{cmplx:outer} with $\eta\le 1/(Ld)$, we have the number of calls to RGO is \eqref{cmplx:mu-smooth}.
Moreover, it follows from Proposition \ref{prop:bundle-smooth} with $\eta\le 1/(Ld)$ that Algorithm \ref{alg:PBS} has iteration-complexity $\tilde {\cal O}(1)$.
As a consequence, the total  iteration-complexity bound for obtaining $\varepsilon$ total tolerance to $\pi$ in terms of total variation is \eqref{cmplx:mu-smooth}.
\QEDA

\end{document}